\documentclass[sigconf,nonacm,natbib=true]{acmart}
\setcopyright{none}
\renewcommand\footnotetextcopyrightpermission[1]{}

\usepackage[utf8]{inputenc}
\usepackage{graphicx}
\usepackage{svg}
\usepackage{amsmath}
\usepackage{amsthm}
\usepackage{thmtools,thm-restate}
\usepackage{multicol, blindtext}
\usepackage[linesnumbered,boxed,noend, noline]{algorithm2e}
\usepackage{multirow}
\usepackage{enumitem}

\newtheorem{assumption}{Assumption}

\usepackage{subfig}
\usepackage[noend]{algpseudocode}
\usepackage{cancel}
\usepackage[dvipsnames]{xcolor}
\usepackage{todonotes}
\algdef{SE}[DOWHILE]{Do}{doWhile}{\algorithmicdo}[1]{\algorithmicwhile\ #1}%

\acmConference[KDD '26]{ACM SIGKDD Conference on Knowledge Discovery and Data Mining}{August 09--13, 2026}{Jeju, South Korea}
\acmISBN{978-1-4503-XXXX-X/2026/08}

\title{Stochastic Primal-Dual Decoding for Multiobjective Generative Recommender Systems}

\author{Dmitrii Moor}
\orcid{0009-0000-4582-7931}
\email{dmitriim@spotify.com}
\affiliation{%
   \institution{Spotify}
   \city{London}
   \country{United Kingdom}
 }

\author{Ben Carterette}
\email{benjaminc@spotify.com}
\affiliation{%
   \institution{Spotify}
   \city{New York}
   \country{USA}
 }

\author{Senthilkumar Krishnamoorthy}
\email{senthilkumark@spotify.com}
\affiliation{%
   \institution{Spotify}
   \city{New York}
   \country{USA}
 }
 
 \author{Kyle Kretschman}
\email{kylek@spotify.com}
\affiliation{%
   \institution{Spotify}
   \city{New York}
   \country{USA}
 }

\author{Denis Beslic}
\email{denisb@spotify.com}
\affiliation{%
   \institution{Spotify}
   \city{New York}
   \country{USA}
 }

\author{Melissa Yalla}
\email{melissamcguirl@spotify.com}
\affiliation{%
   \institution{Spotify}
   \city{New York}
   \country{USA}
 }

\author{Alice Y Wang}
\email{alicew@spotify.com}
\affiliation{%
   \institution{Spotify}
   \city{New York}
   \country{USA}
 }
 
 \author{Mounia Lalmas}
\email{mounia@acm.org}
\affiliation{%
   \institution{Spotify}
   \city{London}
   \country{United Kingdom}
 }

\renewcommand{\shortauthors}{Dmitrii Moor et al.}

\begin{document}

\begin{abstract} 

Recent advances in recommender systems (RS) have shown substantial performance gains through generative modelling. In practice, recommendation often involves constructing slates--ordered lists of items--that must satisfy multiple objectives beyond relevance, such as constraints defined over item attributes or fairness constraints. 
Existing multiobjective approaches either rely on post-processing techniques designed for non-generative settings, or incorporate auxiliary objectives directly into model training. 
The former does not explicitly account for the sequential nature of generative RS, while the latter is often impractical in large-scale systems.

We propose a lightweight, inference-time decoding layer that augments autoregressive generative RS to support multiobjective slate generation without modifying or retraining the underlying model.
We formulate decoding as an online constrained optimisation problem, where items are selected sequentially, and trade-offs between relevance and auxiliary objectives are adjusted dynamically based on the remaining constraint slack, i.e., how much of each objective remains to be satisfied. 
This is implemented via a stochastic primal-dual approximation scheme that balances relevance and auxiliary objectives during generation.

We provide theoretical guarantees on constraint violation and regret, and evaluate the proposed approach through extensive offline experiments and a large-scale online A/B experiment in a real-world recommender system. Our results show consistent improvements in multiobjective trade-offs, including a +1.8\% gain in the auxiliary objectives achieved at zero cost to user satisfaction.

\end{abstract}



\maketitle

\section{Introduction} 

Recent advances in generative modelling have led to remarkable progress in text generation, enabling applications such as ChatGPT that have transformed how people interact with information~\cite{Brown_2020}. 
This success has motivated a growing body of work exploring generative models in recommender systems (RS). An increasing number of studies~\cite{sasrec_paper, transact, BERT4Rec, Raffel_2020, Zhai_2024, FISSA, Deldjoo_2024} investigate the construction of recommendation slates using deep autoregressive (AR) generative models, typically parameterised by transformer-based architectures~\cite{Vaswani_NIPS2017}.
These models factorise user-item interactions into sequences of conditional probabilities, enabling slates to be generated autoregressively based on a user’s interaction history. Empirical evidence from offline and online evaluations shows that such approaches can significantly improve user satisfaction in RS~\cite{transact, xia2025transactv2lifelonguser, sasrec_paper}.



In real-world RS, recommendation slates must satisfy multiple objectives beyond user satisfaction. In addition to engagement, systems often need to account for considerations such as fairness, diversity, or constraints defined over item attributes, such as requiring a minimum presence of items from different categories or with specific novelty properties. 
This requires careful coordination to meet multiple objectives while maintaining a high-quality user experience \cite{Anderson_2020, Mehrotra2021}.

These challenges are well recognised in the RS community~\cite{Zheng_2021}. 
Existing approaches to multiobjective recommendation generally follow two directions. A large body of work addresses multiobjective optimisation in non-generative settings, relying on post-hoc re-ranking or scalarisation techniques that assume item independence~\cite{Mehrotra2018, Mehrotra2021, Mehrotra2020}. More recently, several studies have incorporated auxiliary objectives directly into the training of generative models~\cite{gao2025futureconditionedrecommendationsmultiobjectivecontrollable, Zhang_2024}. While effective in their respective settings, these approaches involve trade-offs when applied to autoregressive generative RS. In particular, non-generative methods fail to capture sequential dependencies~\cite{Meng_2025}, while approaches that encode objectives directly into model parameters become inflexible when objectives or item attributes change and retraining is costly.



In this work, we complement these approaches by introducing a lightweight decoding strategy that augments a wide range of autoregressive generative RS. The proposed decoder operates entirely at inference time and can be applied to existing deployed models without modifying their training procedure. It extends the final scoring stage of transformer-based models trained on user-item interaction data~\cite{sasrec_paper, Zhai_2024, transact} by combining relevance logits with item-level objective signals, guiding the autoregressive selection of items toward desired trade-offs.
This realises multiobjective control directly at the level of decoding, enabling step-wise, constraint-aware adjustment of objective trade-offs as the slate is generated.

To formalise this decoding mechanism, we adopt an \textit{online} constrained optimisation perspective. 
We then view generative decoding as a sequential decision process in which items are selected one at a time and multiple objectives must be satisfied over the course of constructing a recommendation slate. 
In this view, the generative model defines an evolving optimisation \emph{environment} through step-wise relevance scores. The decoder tracks the \emph{slack} of each objective--the remaining gap to its target--and adjusts objective weights accordingly: Objectives with larger slack are prioritised, while those close to satisfaction are gradually down-weighted. This adaptive reweighting is implemented via Lagrangian multipliers updated online using a primal-dual formulation.


The proposed primal-dual decoding strategy provides a principled mechanism for balancing multiple objectives with theoretical robustness guarantees. It is particularly well suited to the generative setting, where relevance distributions evolve across decoding steps, and it exhibits predictable behaviour that supports downstream control mechanisms such as pacing and capacity management~\cite{Moor_KDD_2025}.
These properties are especially important for the practical deployment of complex transformer-based RS in production environments. 
 We evaluate our approach through extensive offline experiments across multiple large-scale RS datasets and a large-scale online experiment, demonstrating consistent Pareto improvements over existing baselines.

 
Our key contributions are as follows:

\begin{itemize}
     \item \textbf{Inference-Time Decoding Method:} We formulate multiobjective recommendation in the generative setting from an online constrained optimisation perspective and derive a lightweight decoding strategy that augments the final scoring stage of autoregressive RS.
    \item \textbf{Theoretical Guarantees:} We provide a theoretical analysis characterising constraint satisfaction and regret as a function of algorithmic hyperparameters. 
    \item \textbf{Empirical Evaluation:} We conduct extensive offline and online experiments across multiple tasks and architectures, demonstrating consistent improvements in multiobjective trade-offs without degrading user experience.
\end{itemize}

The remainder of the paper is organised as follows: Section~\ref{related_work} reviews related work. 
Section~\ref{sec:formal_model} formalises the problem and introduces the online optimisation framework. 
Section~\ref{sec:primal_dual} presents the stochastic primal-dual decoder and its theoretical analysis.
Section~\ref{sec:evaluation} reports experimental results, and  Section~\ref{sec:conclusion} concludesr.

\section{Related Work}\label{related_work}

We review prior work on generative RS and multiobjective recommendation, highlighting connections to decoding-based approaches.

\paragraph{Generative Modelling in RS}
Kang et. al.~\cite{sasrec_paper} proposed a causal transformer-based model optimised on sequential interaction data for next-item prediction, demonstrating significant gains in positive user interactions compared to non-generative models.
This modelling paradigm was later adopted for next-action prediction in~\cite{transact, xia2025transactv2lifelonguser}, where online improvements were reported in large-scale deployments.

Several follow-up works~\cite{FISSA, Liwei_2020, Zhai_2024, PetrovMacdonald2023GPTRec} extended these ideas to better capture global and local user preferences as well as long interaction sequences, showing promising online gains when applying autoregressive (AR) generation in large-scale RS~\cite{Zhai_2024}.
More recently, Geng et al.~\cite{P52022} demonstrated that fine-tuned language models (LLMs) can also be applied for recommendation tasks.
Similarly to  the approaches above, these models rely on causal transformer architectures, enriched with natural language representations (see also~\cite{Sanner_2023, mao-etal-2023-unitrec, wang2024rethinkinglargelanguagemodel}).
In this work, we design a decoding layer that is compatible with all of the above generative architectures.  

Recently, Tomasi et. al.~\cite{Tomasi_diffusion} explored diffusion models for prompt-conditioned slate generation.
While these models show promise, they currently underperform AR approaches in RS settings~\cite{Sahoo_discrete_diffusion}; we therefore focus on AR models.
For a comprehensive survey of generative modelling in RS, we refer the reader to~\cite{Deldjoo_2024}. 

\paragraph{Multiobjective Recommendations}
A large body of work focuses on balancing user satisfaction with additional objectives such as fairness, diversity, or creator-level constraints. 
Many approaches adopt simple scalarisation strategies, combining multiple objectives into a single target via a weighted average  \cite{Zhang_2014, Mehrotra2018}.

Beyond scalarisation, several works have explored sequential decision-making formulations for multiobjective optimisation. Busa-Fekete et al.~\cite{MO_bandits_2017} proposed a bandit-based framework for multiobjective optimisation, later extended to contextual recommendation settings in~\cite{Mehrotra2020}. 
While these methods provide principled mechanisms for learning trade-offs between objectives, they are typically developed for non-generative ranking settings and do not explicitly account for sequential dependencies in slate generation.

More recently, sequence-model-based approaches have incorporated multiobjective control directly into generative models. 
The Decision Transformer architecture~\cite{Decision_transformer} conditions generation on desired outcome signals, and Gao et al.~\cite{gao2025futureconditionedrecommendationsmultiobjectivecontrollable} extend this idea to a multiobjective setting by conditioning on multiple target objectives. 
While effective, such approaches integrate objective control at training time, requiring retraining or fine-tuning when item attributes change, which is costly in large-scale RS deployments.
In contrast, our work focuses on decoder-level techniques that enable multiobjective control at inference time, allowing objectives to be adjusted dynamically without modifying model parameters.

\paragraph{Constraint Optimisation \& Constrained Decoding}
Prior work on constrained decoding has explored optimisation-based and search-based methods for controlling sequence generation. 
Wu et al.~\cite{constrained_decoding_sigir_2025} investigate corpus-specific constraints and beam search design for large language models, while related beam search–based heuristics are studied in~\cite{Hokamp2017LexicallyCD, info12090355}.
In contrast to these heuristic approaches, we focus on deriving a principled approximation algorithm grounded in online constrained optimisation.


\begin{figure}[t]
\includegraphics[width=.85\linewidth]{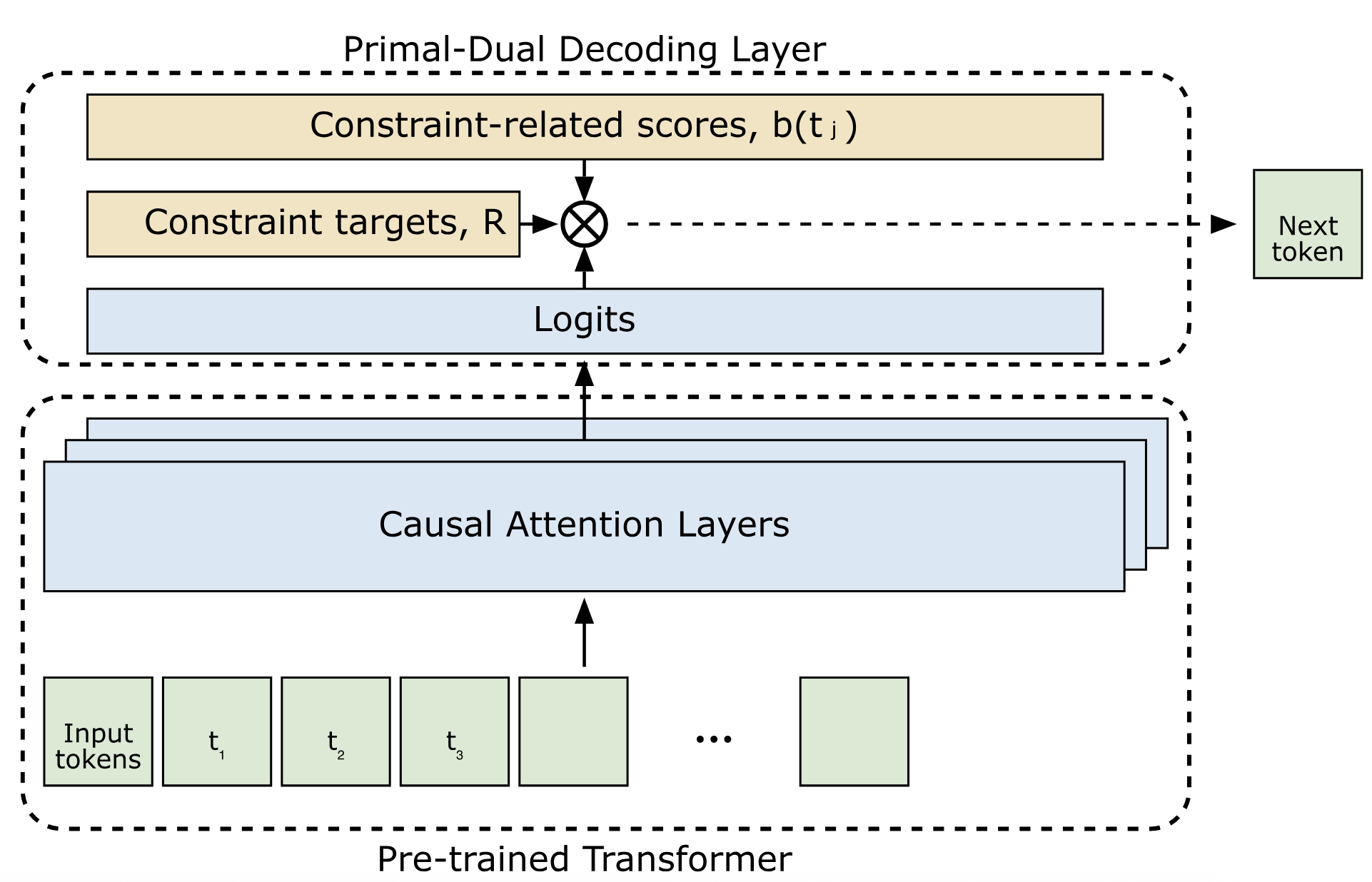}
\vspace{-5pt}
\caption{\small{Stochastic Primal-Dual Decoder: The decoding layer relies on the scores output by the pre-trained transformer $p_{\theta}$ and the constraint-related scores to steer generation into the desired direction.}}
\label{fig:structure}
\end{figure}

\section{Generative Recommender System as Online Optimisation}\label{sec:formal_model}

We begin by introducing the formal model and framing generative slate construction from an online constraint optimisation view. 

Let $T=\{t_1, ..., t_M\}$ denote the set of  \textit{items} (or ``tokens") and let $x \in T^*$ denote a sequence of items representing the \textit{prompt}. 
We consider a parametric generative model $p_{\theta}: T^* \rightarrow \Delta(T)$, 
which maps a prompt $x$ to a probability distribution over items, where $\theta$ denotes the model parameters.

In RS applications, items correspond to entities such as music tracks or movies, while prompt $x$ represents a user's \textit{history} of interactions. 
We assume $p_{\theta}$ is trained on sequences of past user-item interactions~\cite{sasrec_paper, transact, xia2025transactv2lifelonguser}, and thus outputs a distribution over the next item a user is likely to interact with, conditioned on history $x$. 

We focus on the problem of constructing personalised recommendation \textit{slates} of fixed length $K$ (e.g., playlists of $K$ tracks or shelves of $K$ movies)~\cite{Boutilier_2010}. 
To generate such a slate, the model $p_{\theta}$ selects items \emph{autoregressively}: At each step, an item is sampled or selected from the model’s output distribution and appended to the prompt, which is then used to generate the next item~\cite{sasrec_paper, transact}.


Throughout this work, we adopt the common assumption that the scores produced by $p_{\theta}$ reflect the relevance of items for the user. Accordingly, we denote by $p_{\theta}(t_j | x, t_{1:k})$ the \textit{relevance score} assigned by the model to item $t_j$ given the original prompt $x$ and the sequence of previously selected items $t_{1:k}=(t_1, ..., t_{k})$. 

\subsection{Online Allocation with a Single Objective} 

We begin by considering the single objective setting, which highlights the challenges introduced by autoregressive generation and motivates an online optimisation view of slate construction.

In contrast to the non-generative RS settings~\cite{Mehrotra2018, Mehrotra2020, Mehrotra2021}, where item relevance scores are typically assumed to be independent, AR generative models produce relevance scores that depend on previously selected items. 
In particular, at decoding step $k$, the relevance score $p_{\theta}(t_j | x, t_{1:k-1})$ assigned to an item $t_j$ may differ from the score $p_{\theta}(t_j | x, t_{1:k-2})$ assigned at the previous step. 
As a result, offline optimisation strategies used in non-generative settings~\cite{Mehrotra2018, Mehrotra2021} become inefficient when applied to generative slate construction.\footnote{ 
Intuitively, this inefficiency arises because relevance scores change as items are selected, making precomputed rankings suboptimal~\cite{Moor_CIKM23}.} 



To address this, we model slate construction from an \textit{online} optimisation perspective. 
We view the generative model $p_{\theta}$ as defining an optimisation \textit{environment} that evolves as items are selected (see Figure~\ref{fig:structure}).
At each step $k=1,...,K$, the environment reveals scores $p_{\theta}(t_j | x, t_{1:k-1})$ for all candidate items $t_j, j=1,...,M$. 
Based on these scores, the algorithm selects one previously unallocated item to occupy position $k$, while keeping  past allocations fixed.
Formally, the resulting online allocation problem can be written as:
\begin{align}\label{eq:single_objective}
    \max_{a_{kj}\in\{0,1\}} &\sum_{k=1}^K \sum_{j=1}^M p_{\theta}(t_j | x, t_{1:k-1})a_{kj}\\
    &\sum_{k=1}^K a_{kj} \leq 1\;\;\forall j=1,...,M\label{eq:feasibility_constr_0}\\
    &\sum_{j=1}^M a_{kj} = 1\;\;\forall k=1,...,K\label{eq:feasibility_constr_01}.
\end{align} 
Here, $a_{kj}=1$ indicates that $t_j$ is assigned position $k$, and $a_{kj}=0$ otherwise. 
Objective (1) maximises the total expected relevance, Constraint~(\ref{eq:feasibility_constr_0}) and (\ref{eq:feasibility_constr_01}) ensure that each item is selected at most once and that exactly one item is assigned to each position respectively.

Since the problem is solved online, at any step $k$ only the relevance scores revealed up to that point are available, and only the variables $a_{k,1:M}$  can be set. 
In practice, solving (\ref{eq:single_objective})-(\ref{eq:feasibility_constr_01}) optimally is computationally intractable~\cite{Moor_CIKM23}, and  heuristic approaches such as greedy decoding are typically used to get approximate solutions.


\subsection{Multiobjective Setting}\label{multi-obj-setting}
We now extend the single objective formulation introduced above to a multiobjective setting. In addition to maximising user relevance, we also consider an auxiliary objective defined over item attributes, for example requiring that a minimum number of items belonging to a given category or satisfying a specific attribute constraint appear in the slate. For clarity of exposition and without loss of generality, we focus on the case of a single additional constraint.

We associate each item $t_j$ a non-negative value $b(t_j)$, which quantifies its contribution to the auxiliary objective (e.g., a binary novelty score or a continuous contribution weight). We further define a \textit{target} $R\geq 0$ representing the minimum level that this objective must achieve over the entire slate. 
The goal is to construct a slate that maximises total relevance while ensuring that the cumulative contribution of the selected items meets this target. 
%
Formally, the multi-objective allocation problem can be written:

\begin{align}
     \max_{a_{kj}\in\{0,1\}} &\sum_{k=1}^K \sum_{j=1}^M p_{\theta}(t_j | x, t_{1:k-1}) a_{kj}\label{eq:max_usat_online_1}\\
     \text{s.t.,} &\sum_{k=1}^K \sum_{j=1}^M b(t_j) a_{kj} \geq R\label{eq:budget_constr}\\
     &\sum_{k=1}^K a_{kj} \leq 1\;\;\forall j=1,...,M\label{eq:feasibility_constr_11}\\
     &\sum_{j=1}^M a_{kj} = 1\;\;\forall k=1,...,K.\label{eq:feasibility_constr_1}
\end{align} 
As in the single-objective case, the relevance scores $p_{\theta}(t_j | x, t_{1:k-1})$ in Objective~(\ref{eq:max_usat_online_1}) are not known in advance, but are revealed sequentially by the generative model as decoding progresses. 
Constraint~(\ref{eq:budget_constr}) enforces that the total contribution of the selected items toward the auxiliary objective reaches the target $R$,  while Constraint~(\ref{eq:feasibility_constr_11}) ensures that each item is selected at most once.

Problem~(\ref{eq:max_usat_online_1})-(\ref{eq:feasibility_constr_1}) defines the desired behaviour of the decoding process when generating recommendation slates under multiple objectives. Importantly, the optimal trade-off between relevance and the auxiliary objective depends on both past allocation decisions and the remaining capacity to satisfy the constraint over future positions in the slate. This observation motivates the need for inference-time control mechanisms that can adapt dynamically during generation. In the next section, we derive such a decoding strategy using a primal-dual approximation framework.

\section{Online Optimisation-Based Inference}\label{sec:primal_dual} 

We now derive an online decoding strategy that  solves the multi-objective allocation problem defined in Section~\ref{multi-obj-setting} approximately. The proposed approach leverages a stochastic primal-dual formulation to guide autoregressive generation toward constraint-satisfying recommendation slates. We first present the decoding algorithm, and then analyse its theoretical properties.

\subsection{Primal-Dual Decoding}


We begin by providing intuition for the proposed decoding strategy. During autoregressive generation, items are selected sequentially, and the contribution toward an auxiliary objective (e.g., satisfying attribute-based or diversity-related constraints) accumulates over time. Early in the generation process, the auxiliary objective may require stronger emphasis, whereas later it may already be close to being satisfied. An effective decoder should therefore adaptively adjust the trade-off between relevance and the auxiliary objective based on how much of the target remains to be achieved.

We capture this adaptive trade-off using a primal-dual formulation, where a Lagrangian multiplier dynamically controls the relative importance of the auxiliary objective during decoding. This formulation is particularly well suited to the generative setting, where decisions are made sequentially and future relevance scores are not known in advance.

Under this view, Problem~(\ref{eq:max_usat_online_1})-(\ref{eq:feasibility_constr_1}) can be interpreted as an instance of online constrained optimisation.
During autoregressive generation, the allocation variables $a_{kj}$ are determined sequentially for positions $k=1,...,K$. 
As items are selected one by one, the \textit{slack} of auxiliary Constraint (\ref{eq:budget_constr})--defined as the remaining gap between the accumulated contribution and the target $R$--changes over time. As this slack decreases, satisfying the constraint becomes easier, and the relative importance of prioritising items with high auxiliary contribution diminishes.

This evolving trade-off between relevance and auxiliary contribution can be formalised using a Lagrangian multiplier $\lambda$, which acts as a dynamic weight on the constraint~\cite{Bertsimas1998, MasColell}. Intuitively, larger values of $\lambda$ place greater emphasis on satisfying the auxiliary objective, while smaller values prioritise relevance.


We begin by dualising Constraint (\ref{eq:budget_constr})~\cite{Bertsimas1998}, which transforms the optimisation problem into a form that can be handled incrementally during decoding. 
Problem~(\ref{eq:max_usat_online_1})-(\ref{eq:feasibility_constr_1}) can be rewritten as:
\begin{align}
     \max_{a_{kj}\in\{0,1\}} \Big[ &\sum_{k=1}^K \sum_{j=1}^M p_{\theta}(t_j | x, t_{1:k-1}) a_{kj} + \min_{\lambda\geq 0}  \Big(\lambda \sum_{k=1}^K \sum_{j=1}^M b(t_j) a_{kj} - \lambda R   \Big) \Big]\label{eq:dualised}\\
     \text{s.t.,}\;\;&\sum_{k=1}^K a_{kj} \leq 1\;\;\forall j=1,...,M,\\
     &\sum_{j=1}^M a_{kj} = 1\;\;\forall k=1,...,K.\label{eq:feasibility_constr_2}
\end{align} 
Here, $\lambda \geq 0$ is the Lagrangian multiplier associated with the auxiliary constraint, and we let $\lambda_k$ denote its value at decoding step $k$. 
For a fixed $\lambda_{k}$, maximising the outer objective in (\ref{eq:dualised}) reduces to selecting, at each step $k$, the item that maximises a weighted combination of relevance and auxiliary contribution:\footnote{We provide the respective derivations in Appendix \ref{app:proofs}}
\begin{align}\label{eq:outer_problem}
    j^* \in \arg\max_{j} \Big\{p_{\theta}(t_j | x, t_{1:k-1}) + \lambda_{k} b(t_j)  \Big\}.
\end{align}
Once an item is selected, the realised contribution toward the auxiliary objective becomes known, and the constraint slack changes. 
Consequently, the multiplier $\lambda$ must be updated after each allocation decision to reflect the updated state of constraint satisfaction. 
This update is obtained by solving the inner minimisation problem in (\ref{eq:dualised}):\footnote{Here, we substituted $R=\sum_{k=1}^K \frac{R}{K}$, and used the fact that each position of the slate must be allocated, i.e., $\frac{R}{K}=\sum_{j=1}^M \frac{R}{K} a_{kj}$ for all $k=1,...,K$, see Equation (\ref{eq:feasibility_constr_01}).}
%
%
\begin{align}\label{eq:inner_problem}
    \lambda_{k+1} = \arg\min_{\lambda \geq 0}\Big\{ \lambda \sum_{k=1}^K \sum_{j=1}^M \Big( b(t_j) - \frac{R}{K}\Big) a_{kj} \Big\}.
\end{align}
The term $R/K$ represents the per-position share of the target, allowing the update to compare the realised contribution at each step against the average contribution required to meet the constraint over the full slate.

Since this is a continuous optimisation problem and updates must be performed online, we update $\lambda$ using \textit{mirror descent}~\cite{Nemirovsky_1985, Balseiro2020DualMD}. Mirror descent is well suited to this setting as it enforces non-negativity and provides stable, multiplicative updates commonly used in online allocation and pacing problems.
Concretely, we solve:
\begin{align}\label{eq:inner_problem_1}
    \lambda_{k+1} = \arg\min_{\lambda \geq 0}\Big\{ \lambda \sum_{j=1}^M \Big(b(t_j)-\frac{R}{K} \Big) a_{kj}  + \frac{1}{\eta} D_{h}(\lambda || \lambda_k) \Big\},
\end{align}
where $D_h(\lambda || \lambda_k)$ denotes the Bregman divergence induced by the negative entropy function $h(\lambda) = \lambda \ln \lambda - \lambda$, see~\cite{Bubeck_2015, Balseiro2020DualMD}, and $\eta>0$.
This regularisation discourages large, unstable changes in $\lambda$ across consecutive decoding steps.

Now, solving (\ref{eq:inner_problem_1}) yields the closed-form update:
\begin{align}\label{eq: multiplier_update}
    \lambda_{k+1} = \lambda_0 \exp\Big\{ -\eta \sum_{k=1}^K \sum_{j=1}^M \Big(b(t_{j}) - R/K\Big)a_{kj} \Big\},
\end{align}
where $\eta>0$ is a step-size parameter. The exponential form of this update increases $\lambda$ rapidly when the constraint is violated and decreases it smoothly once the constraint is on track, providing stable yet responsive control during generation. A detailed derivation is provided in Appendix~\ref{app:proofs}.


Algorithm \ref{alg:online_opt} summarises our resulting primal-dual decoding procedure. 
The algorithm initialises all allocation variables to zero and sets the initial Lagrangian multiplier to $\lambda_0=1$ (line 1). 
It then generates the recommendation slate autoregressively over $K$ steps (lines 2-7).
At each step $k$, the generative model scores all unallocated items using:
\begin{align}
    \sigma(t_j) = p_{\theta}(t_j | x, t_{1:k-1}) + \lambda_k b(t_j).
\end{align}
and selects the item with the highest score (lines 3-4).
The selected item is added to the slate, and the multiplier $\lambda$ is updated using the rule in (\ref{eq: multiplier_update}), reflecting the updated constraint slack. The process terminates once all $K$ positions have been filled. 

\begin{algorithm}[t]\small
\caption{Online Primal-Dual Allocation Algorithm}\label{alg:online_opt}

\KwData{$\eta, \text{target}\; R,\;\text{bids}\; b(t_j), j=1,...,M$}
Set $\lambda_0 \leftarrow 1$ and $a_{kj} \gets 0$ for all $k=1,...,K,\; j=1,...,M$; $S\gets \emptyset$

\For{$k=1,...K$}
{
  {    
    $\sigma(t_j) \gets p_{\theta}(t_j|x,t_{1:k-1}) + \lambda_{k-1} b(t_j)$ for all $j\in \{1,...,M\} \textbackslash S$
        
    $j^*\gets \arg\max_j \{\sigma_j\}$

    $S \gets S \cup \{j^*\}$

    $a_{ij^*}^k\gets 1;\;\;\; a_{ij}^k\gets 0 \;\;\forall j \neq j^*$

    $\lambda_{k} \gets \lambda_0\exp\{-\eta (\sum_{t=1}^k\sum_{j=1}^M (b(t_j)-R/K) a_{tj}\}$
  }
}
\Return $a_{kj}$, $\forall k,j$
\end{algorithm}
 

\subsection{Theoretical Analysis}\label{sec:theory}

Having introduced the primal-dual decoding strategy in Algorithm~\ref{alg:online_opt}, we now analyse its theoretical properties. 
In particular, we derive bounds on (i) the violation of auxiliary constraint and (ii) the regret with respect to the optimal allocation under Problem~(\ref{eq:max_usat_online_1})-(\ref{eq:feasibility_constr_1}). 
These results provide robustness guarantees for the proposed method and clarify how the step-size hyperparameter $\eta$ in Algorithm~\ref{alg:online_opt} controls the trade-off between relevance maximisation and constraint satisfaction.

A key challenge in analysing Algorithm~\ref{alg:online_opt} arises from the autoregressive nature of generative RS. 
The relevance coefficients $p_{\theta}(t_j | x, t_{1:k-1})$ in Objective (\ref{eq:max_usat_online_1}) are not independent across decoding steps: Selecting an item at step $k$ affects not only the immediate reward, but also the relevance scores observed at future steps. 
This \textit{stochastic} temporal dependence complicates the direct application of standard primal-dual analysis techniques developed for settings with independent or stationary rewards.


To make the analysis tractable, we introduce a mild assumption of the \textit{temporal consistency} of relevance scores:
\begin{assumption}\label{as:temp_cons}
    For any item $j=1,...,M$, iteration $k=1,...,K$ and sequence $t_{1:k-1}$ of previously selected items, we let:
    \begin{align}\label{eq:temp_consistency}
    p_{\theta}(t_j | x, t_{1:k-1}) = p_{\theta}(t_j | x, t_{1:k-2}) + \epsilon_{k-1,j}(x),
\end{align}
where $\epsilon_{k-1,j}(x)\sim \mathcal{N}(0,\sigma_x)$ captures the effect of the previously selected item on the relevance of $t_j$, and the base relevance $p_{\theta}(t_j |x)$ are independent across items $t_j$.
\end{assumption}
Intuitively, this assumption allows relevance to be analysed as consisting of a stable, item-specific component that evolves gradually as items are selected. It reflects the idea that, in well-calibrated generative models, relevance scores may change during autoregressive generation but do so in a controlled manner rather than arbitrarily. Similar forms of temporal consistency assumptions are commonly used in the analysis of sequential decision-making and temporally dependent models (see, for example, ~\cite{maystre2025incrementalsequenceclassificationtemporal}).


Under Assumption~\ref{as:temp_cons}, we can now establish the following guarantee on constraint satisfaction:

\begin{restatable}{statement}{OnlineAlg}\label{th:cv}
    If Problem~(\ref{eq:max_usat_online_1})-(\ref{eq:feasibility_constr_1}) is feasible, then for any $\eta>0$, the constraint violation incurred by Algorithm~\ref{alg:online_opt} is bounded by \[O\Big(\frac{1}{\eta}\log\frac{1}{b_{max}-\frac{R}{K}}\Big).\]
\end{restatable}
\begin{proof}
We provide the proof in Appendix~\ref{app:cv}. 
\end{proof}

Statement~\ref{th:cv} shows that increasing $\eta$ tightens auxiliary constraint (\ref{eq:budget_constr}), leading to  smaller violations.
This behaviour is intuitive: As seen in line 7 of Algorithm~\ref{alg:online_opt}, a larger $\eta$ makes the dual variable $\lambda_k$ more responsive to constraint violations.
As a result, items with higher auxiliary contribution are prioritised more aggressively, 
pushing the allocation  towards meeting the target $R$.

However, enforcing the constraint more strictly comes at a cost. Increasing $\eta$ also increases the emphasis placed on the auxiliary objective, potentially reducing relevance. This trade-off is captured by the following regret bound.


\begin{restatable}{statement}{algRegret}\label{th:alg_regret}
    Regret of Algorithm \ref{alg:online_opt} is upper bounded by
    \begin{align}
        Regret(\mathcal{A}) \leq O\Big(\frac{1}{\eta} + \eta K + \frac{4}{3}\sigma_x K^{3/2} \sqrt{2\log\frac{2MK}{\delta}}\Big)
    \end{align}
    with probability at least $1-\delta$. 
\end{restatable}
\begin{proof}
We provide the proof in Appendix \ref{app:regret}.
\end{proof}

Statement~\ref{th:alg_regret} confirms that larger values of $\eta$ weakens the regret bound, reflecting the increased prioritisation of constraint satisfaction over relevance. Together, Statements~\ref{th:cv} and~\ref{th:alg_regret} characterise a clear trade-off controlled by the step-size parameter $\eta$, which can be tuned in practice to balance relevance and auxiliary objectives.

Finally, we note that Algorithm~\ref{alg:online_opt} has  linear complexity in both the slate length $K$ and the number of candidate items $M$. This makes the proposed decoding strategy well suited for low-latency, large-scale recommender system deployments.


\section{Evaluation}\label{sec:evaluation}
We evaluate our approach through a combination of offline and online experiments. 
Section~\ref{sec:offline_evaluation} describes the offline evaluation setup and present the results from large-scale experiments across multiple datasets and generative architectures. Section~\ref{sec:online_evaluation} reports the outcomes of a large-scale online A/B experiment conducted in a real-world RS. In Section~\ref{sec:hyperparam}, we analyse the impact of hyperparameter choices on the trade-off between relevance and constraint satisfaction. Finally, Section~\ref{sec:discussion} summarises the main findings.

\subsection{Offline Evaluation}\label{sec:offline_evaluation}

We begin by describing our offline evaluation setup and summarising the results of our offline study.
We first introduce the recommendation domains and the corresponding datasets used in our experiments, followed by the transformer-based architectures employed in each domain. We next  outline the decoding strategies used as baselines. 
All decoding methods are applied on top of the same underlying generative models,  enabling consistent comparison across domains and the computation of final offline metrics. 

\begin{figure*}[h]
\includegraphics[width=.85\linewidth]{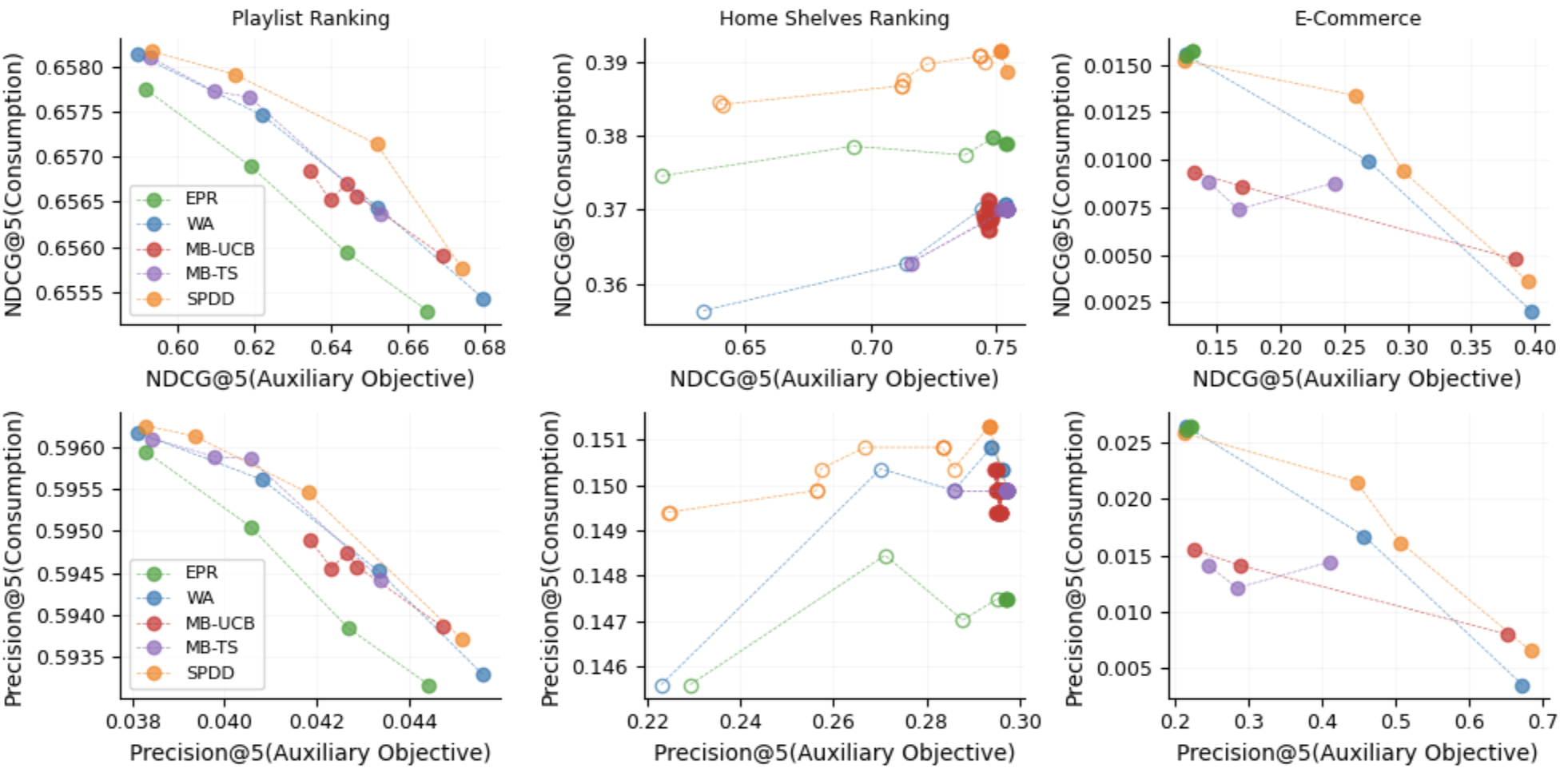}
\caption{Pareto frontiers for consumption vs. auxiliary objective stream share for the three different domains.}\label{fig:pareto}
\end{figure*}

\paragraph{Datasets} We evaluate our approach using two large-scale datasets from a major digital content provider, and one publicly available e-Commerce dataset.  
Across all datasets, we consider an auxiliary objective that requires achieving a target level of allocation for a designated subset of items, corresponding to Constraint (\ref{eq:budget_constr}).\footnote{For sensitivity reasons, we abstract the auxiliary objective in our experiments. In practice, this objective can represent a range of item-level constraints, such as category-level allocation constraints, novelty, or diversity requirements. Owing to the generality of the formulation in (\ref{eq:max_usat_online_1})-(\ref{eq:feasibility_constr_1}), this abstraction does not affect the validity of our evaluation.} 
\begin{itemize}[leftmargin=0.4cm]
    \item \textit{Playlist Ranking}. The dataset contains approximately 40B interactions from 62M users across more than 42M music tracks collected over a two-week period on a large streaming platform. 
   It includes metadata identifying a subset of tracks associated with the auxiliary objective (i.e., coefficients $b(t_j)$ in Constraint (\ref{eq:budget_constr})).
    We use this dataset to generate playlists that maximise user consumption while meeting a target stream share for the designated subset of tracks, as formalised in Problem~(\ref{eq:max_usat_online_1})-(\ref{eq:feasibility_constr_1}). Additional details are provided in Appendix~\ref{app:playlist_dataset}.
    
    \item \textit{Home Shelves Ranking}. The Home page is the primary entry point for user interaction with the platform. It consists of a \textit{cards} (which may refer to entities such as playlists, artists, podcasts, audiobooks) arranged in thematically-coherent \textit{shelves}.
    We obtained a Home training dataset comprising 5.3B shelf impressions from over 500M distinct user requests collected over one week, along with metadata identifying shelves associated with the auxiliary objective. 
    As in the playlist setting, we evaluate models that rank shelves to maximise user engagement subject to an impression target on this subset (see Appendix~\ref{app:home_shelves}). 
    
    \item \textit{E-Commerce Recommendations}. We additionally use \textit{Sports and Outdoors} subset of the publicly available Amazon e-Commerce dataset~\cite{amazon_data}, containing over 5M user–item interactions across more than 1M items. 
    We selected two item sub-categories of items--\textit{`Camping \& Hiking'} and \textit{`Cycling'}--to define the auxiliary objective.
    Further details can be found in Appendix~\ref{app:amazon}.    
\end{itemize}

\paragraph{Generative Models} To demonstrate the generality of the proposed inference-time decoding strategy, we evaluate it across a diverse set of widely used generative RS models. In all cases, the underlying model is trained in a single-objective setting to optimise relevance, and the proposed primal-dual decoding layer is applied only at inference time to enable multiobjective control: 
\begin{itemize}[leftmargin=0.4cm] 
    \item \textit{SASRec} \cite{sasrec_paper} is a causal transformer architecture optimised on sequential user-item interaction data for next-item prediction. We use this model for the \emph{Playlist Ranking} task.
    
    \item \textit{TransAct} \cite{transact} is a transformer-based model designed for real-time user action prediction and has demonstrated strong performance in large-scale production systems. 
    We apply this model to the \emph{Home Shelves Ranking} task.
    
    \item \textit{LM-Rec} \cite{P5_2022} is a pre-trained LLaMA-based language model fine-tuned on the e-commerce dataset for next-item recommendation in the spirit of \cite{P5_2022}. We use this model to evaluate our approach in an open-domain setting that relies on semantic item representations rather than conventional item identifiers.
\end{itemize}


For each model and domain, we first generate slates using the base model optimised for relevance alone. 
We then apply different decoding strategies, including the proposed primal-dual decoder, on top of the same model outputs to incorporate auxiliary objectives at inference time. 
This design isolates the effect of the decoding layer and enables a fair comparison across models and tasks. 

\paragraph{Baselines} 
We compare the offline performance of the proposed {\bf Stochastic Primal–Dual Decoder} ({\bf SPDD} for short) (Algorithm~\ref{alg:online_opt}) against several decoding-based approaches that introduce multiobjective control at different stages of the generation process.
\begin{itemize} [leftmargin=0.4cm]
    \item \textit{Ex-post Re-ranking (EPR)}. The baseline follows the approach of~\cite{Mehrotra2018} and applies multiobjective optimisation \emph{after} slate generation. 
    Item-level relevance scores $p_{\theta}(t_j | x, t_{1:k-1})$ are combined with auxiliary scores $b(t_j)$ using a simple weighted average once all relevance scores for the slate have been computed. Items are then ranked according to the combined score, ignoring the sequential nature of generative decoding.
    
    \item \textit{Online Weighted Average (WA)}. This baseline applies a similar weighted average heuristic as EPR, but does so \emph{during} autoregressive generation. At each decoding step, relevance and auxiliary scores are combined using fixed weights to select the next item. While this enables online control,  the weights remain fixed and do not adapt based on constraint satisfaction. 
    
    \item \textit{Multiobjective Bandits (MB-UCB)}. This baseline combines relevance and auxiliary scores using the Generalised Gini Index~\cite{MO_bandits_2017, Mehrotra2020}. The weights in the GGI formulation are treated as arms in a bandit problem and are learned online using an Upper Confidence Bound (UCB) strategy to balance multiple objectives.
    

    \item \textit{Multiobjective Thompson Sampling (MB-TS)}. Similar to MB-UCB, this baseline uses Thompson sampling instead of UCB to learn the optimal set of objective weights within the GGI framework.
    
    
    \item \textit{Stochastic Primal-Dual Decoder (SPDD)}. Our proposed decoding strategy, defined in Algorithm~\ref{alg:online_opt}, dynamically balances relevance and auxiliary objectives during autoregressive generation using a stochastic primal-dual update based on constraint slack.
\end{itemize}

We do not include methods that incorporate auxiliary objectives directly into the pre-training or fine-tuning of the generative model (e.g.,~\cite{gao2025futureconditionedrecommendationsmultiobjectivecontrollable}), as such approaches require repeated retraining of large-scale models when objective specifications or item attributes change. Our focus is therefore on decoder-level techniques that enable flexible, inference-time multiobjective control. 


\paragraph{Results}
Figure \ref{fig:pareto} summarises the results of our offline evaluation.
To characterise the trade-offs between user consumption and the stream share of the designated content subset achieved by different decoding strategies, we report NDCG@k (top row) and Precision@k (bottom row).
User consumption is estimated using the total number of interacted items in the generated slates, while the content share is computed as the fraction of interacted items associated with the auxiliary objective.

For each method, we vary the optimisation target (parameter $R$ in Constraint (\ref{eq:budget_constr}) for SPDD, or the scalarisation weight for EPR, WA and MB) and re-rank the logged interaction data accordingly. 
We then recompute NDCG@k and Precision@k separately for user consumption and the auxiliary objective, yielding a single operating point. By repeating this evaluation for different parameter settings, we obtain the Pareto frontiers shown in~Figure~\ref{fig:pareto}.\footnote{In Home Shelves Ranking domain, the auxiliary objective was correlated with the users' consumption, see Figure \ref{fig:pareto} (centre). 
In this case, we also illustrate a few obtained Pareto sub-optimal outcomes with the hollow markers.}


As shown in Figure~\ref{fig:pareto} (left), in the domain with strong sequential dependencies EPR (\textcolor{Green}{green}) is dominated by other multiobjective approaches. This behaviour is expected, as EPR ignores the sequential structure of generative models and applies re-ranking only \textit{after} all relevance scores $p_\theta (t_j | x, t_{1:k-1})$ have been computed.
As a result, EPR performs relatively better only for low constraint targets, where the induced perturbations to the ranking are minimal.


EPR performs particularly poorly for e-commerce (Figure~\ref{fig:pareto}, right), where it fails to achieve high levels of auxiliary-objective satisfaction. In this setting, the designated content subset constitutes a small fraction of available items (see Appendix \ref{app:amazon}), and models optimised solely for relevance rarely select such items. Consequently, \textit{ex-post} re-ranking has limited capacity to improve the corresponding content share.

In contrast, the online weighted-average baseline WA (\textcolor{blue}{blue}) yields a nearly uniform Pareto improvement over EPR across domains, dominating it both in terms of Precision and in terms of NDCG for Playlists/e-Commerce domains. 
However, when compared to the bandit-based approaches MB-UCB (\textcolor{red}{red}) and MB-TS (\textcolor{Purple}{purple}), we observe no consistent additional gains. 
While these methods adapt weights online, their exploration overhead limits practical improvements over fixed-weight strategies.



Finally, the proposed SPDD (\textcolor{orange}{orange}) achieves a substantial and a nearly uniform improvement over all baselines.
For almost any constraint target (x-axis), SPDD attains higher user consumption while satisfying the corresponding auxiliary objective. These improvements are consistent across all three evaluation domains.

\paragraph{Per-Position Analysis}
To better understand the source of the efficiency gains achieved by SPDD, Figure~\ref{fig:positions} presents a per-position analysis of the distribution of items associated with the auxiliary objective and the corresponding interaction volume across the top $K$ positions of the generated slates.

As shown in Figure~\ref{fig:positions} (left), SPDD allocates a slightly larger number of such items at the top position of the slate compared to other methods. Importantly, this does not lead to a noticeable change in user consumption at that position (Figure~\ref{fig:positions}, right). 
At subsequent positions, however, SPDD reduces the share of these items more aggressively, which is associated with increased consumption further down the slate. 
This behaviour reflects the adaptive nature of the proposed decoder, which prioritises satisfying the auxiliary objective early and then shifts focus back to maximising user consumption once the target $R$ has been met. 

\begin{figure}[t]
\includegraphics[width=.99\linewidth]{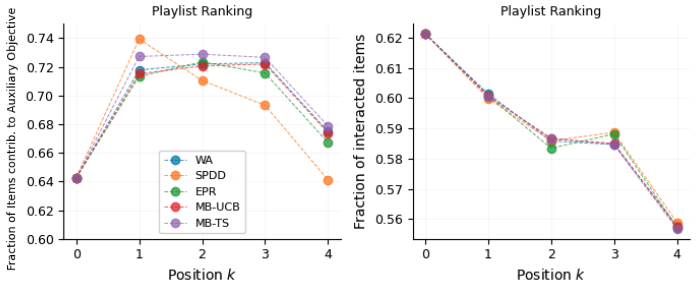}
\caption{\small{SPDD allocates more items with higher contribution towards the auxiliary objective (left) at position 1 without affecting consumption much (right). It then drops the auxiliary objective share while improving consumption.}}\label{fig:positions}
\end{figure}

\begin{figure}[t]
\includegraphics[width=.99\linewidth]{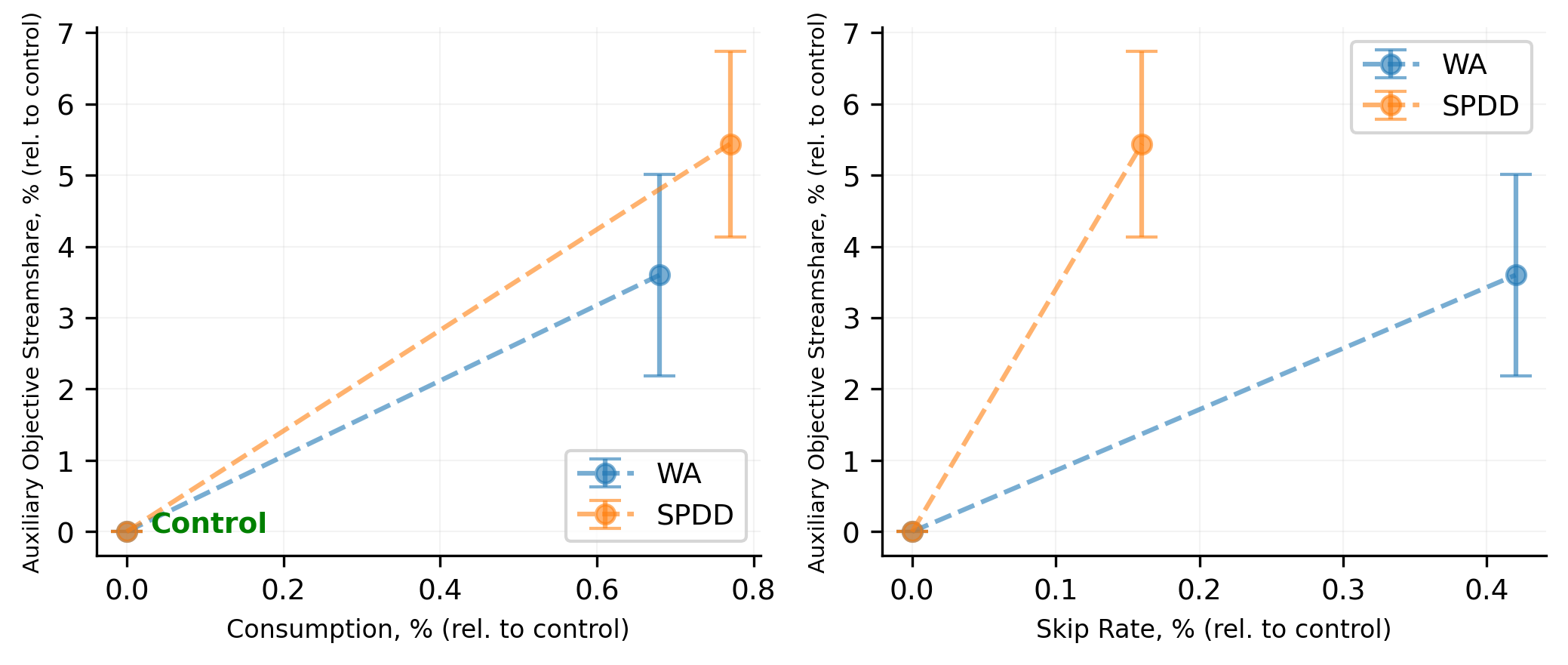}
\vspace{-10pt}
\caption{\small{Online evaluation: SPDD delivers a larger stream share for the auxiliary objective compared to WA while improving consumption (left) and decreasing skip rate (right).}}\label{fig:online_eval}
\end{figure}


\subsection{Online A/B Experiment}\label{sec:online_evaluation}

To evaluate the performance of our approach in a real-world setting, we conducted a large-scale online A/B experiment involving approximately 1M users across more than 150 countries. The experiment re-ranked 600K personalised playlists, with the goal of maximising user consumption (Objective~\ref{eq:max_usat_online_1}) while satisfying a target on the stream share of a designated content subset (Constraint~\ref{eq:budget_constr}).


Users were randomly assigned to one of three groups: (1) a {\bf Control} group served with an autoregressive SASRec model (Section~\ref{sec:offline_evaluation}) optimised solely for consumption;  (2) {\bf Treatment WA}, which used the same autoregressive SASRec model with weighted-average WA decoding strategy (Section~\ref{sec:offline_evaluation}); and (3) {\bf Treatment SPDD}, which combined the same base model with the proposed stochastic primal–dual decoder (Algorithm~\ref{alg:online_opt}).
Baselines EPR, MB-UCB and MB-TS were excluded from the online experiment due to their weaker offline performance. All metrics are reported relative to the control group 
(Figure~\ref{fig:online_eval}).


As shown in Figure~\ref{fig:online_eval} (left), SPDD achieves a higher stream share for the designated content subset (+5.44\%) compared to WA (+3.60\%), without degrading user consumption, measured as total minutes played. Moreover, the skip rate under SPDD is lower than that of WA (+0.16\% vs. +0.42\%, Figure~\ref{fig:online_eval}, right).
Together, these results indicate that the proposed primal–dual decoder delivers a Pareto improvement over fixed weighted-average decoding when optimising multiple objectives in the generative setting.

\subsection{Regret vs. Constraint Violation Trade-off}\label{sec:hyperparam}
We illustrate the trade-off between regret reached by Algorithm~\ref{alg:online_opt} and the corresponding level of constraint violation, as characterised in Statements~\ref{th:cv} and~\ref{th:alg_regret}. Beyond validating the theoretical analysis, this experiment provides practical guidance on selecting the step-size hyperparameter $\eta$ of Algorithm~\ref{alg:online_opt}.

\begin{figure}[t]
\includegraphics[width=.6\linewidth]{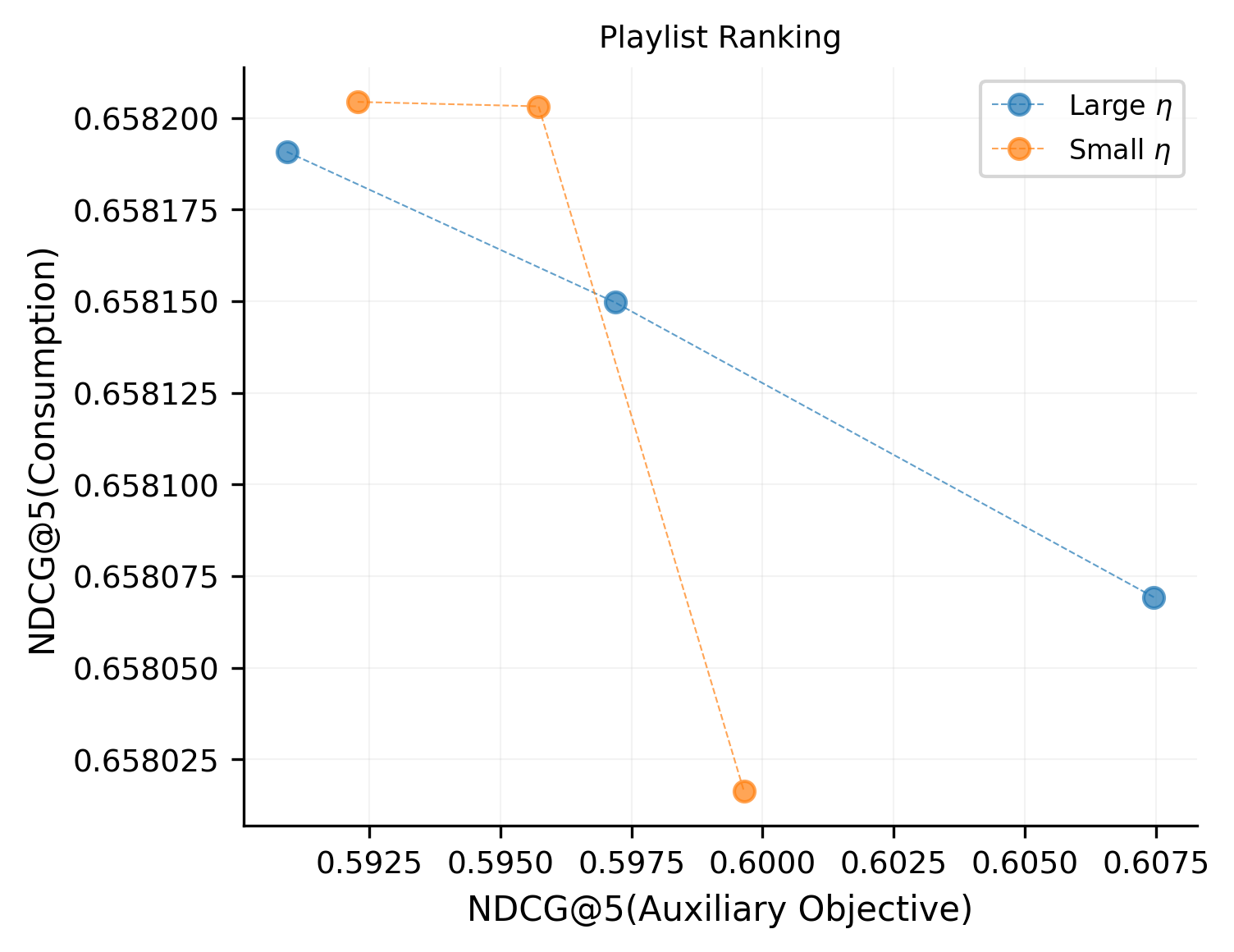}
\caption{\small{Regret vs. Constraint Violation Trade-off: Using small $\eta$ for low constraint targets (x-axis) helps to increase consumption. If the constraint target is large, it may be hard to reach it with small $\eta$.}}\label{fig:reget}
\end{figure}

Figure \ref{fig:reget} depicts this trade-off for the \textit{Playlist Ranking} domain by comparing two Pareto frontiers obtained using different values of $\eta$. 
The \textcolor{orange}{orange} frontier corresponds to a smaller step size, while the \textcolor{blue}{blue} corresponds to a larger step size.
For lower constraint targets (x-axis) the orange frontier dominates, achieving higher user consumption.
This behaviour is consistent with the analysis in Section~\ref{sec:theory}: When the constraint is relatively easy to satisfy, larger values of $\eta$ unnecessarily prioritise constraint satisfaction and lead to higher regret (Statement \ref{th:alg_regret}).

Conversely, for larger constraint targets, the blue frontier dominates. In this regime, a smaller step size is insufficient to enforce the constraint effectively, resulting in higher violation (Statement~\ref{th:cv}).  These results illustrate the trade-off governed by $\eta$ and highlight the importance of tuning this parameter based on the desired balance between relevance maximisation and constraint satisfaction.

\subsection{Summary}\label{sec:discussion}

To assess the effectiveness of the proposed primal-dual multiobjective decoder, we conducted a series of offline and online experiments. The offline results demonstrate a nearly uniform Pareto improvement over all baselines across multiple recommendation domains. Moreover, these gains are consistent across several widely used generative recommender system architectures, highlighting the generality of the proposed approach.


To better understand the source of these improvements, we analysed the generated slates at a per-position level. This analysis shows that the adaptive nature of the decoder enables it to prioritise satisfying the auxiliary constraint early during slate construction and to refocus on maximising user relevance as soon as the constraint target has been met. As a result, potential trade-offs introduced by enforcing the auxiliary objective are mitigated, even in settings where item relevance evolves throughout the autoregressive generation process.


Finally, to assess practical viability, we conducted a large-scale online A/B experiment. The results confirm that the proposed decoding strategy improves multiobjective trade-offs in a real-world deployment without adversely affecting user engagement.

\section{Conclusions}\label{sec:conclusion}

Generative RS introduce a fundamental tension between sequential relevance optimisation and the need to satisfy additional, slate-level objectives. 
Unlike non-generative ranking settings, AR generation induces strong dependencies between selection decisions, rendering post-hoc re-ranking strategies ineffective and making fixed-weight scalarisation brittle. At the same time, incorporating multiobjective control directly into model training tightly couples objectives to model parameters, limiting flexibility in environments where item attributes or optimisation targets evolve over time.

In this work, we addressed this gap by framing generative slate construction as an online constrained optimisation problem and introducing a near-optimal inference-time decoding mechanism for multiobjective control. Our approach treats decoding as the optimisation step and leverages a stochastic primal-dual approximation to adaptively balance relevance and auxiliary objectives during generation. Crucially, this design enables dynamic control without modifying or retraining the underlying generative model.


We illustrated how the derived theoretical guarantees on constraint satisfaction and regret assist with hyperparameter optimisations and make our approach robust to the non-stationary relevance distributions inherent in AR generation. 
Empirically, extensive offline experiments across multiple domains and model architectures show consistent Pareto improvements over existing decoder-level baselines. These results are further corroborated by a large-scale online experiment, which achieved up to a 1.8\% improvement in the auxiliary objective without degrading user engagement.

Taken together, our findings highlight inference-time, optimisation-driven decoding as a principled and practical solution to multiobjective control in generative RS, addressing a key limitation of existing approaches while being compatible with  large-scale deployments.\\

Looking ahead, the proposed framework naturally extends to several directions. 
First, our approach can be adapted towards using personalised user-level auxiliary targets allowing balancing of long-term user satisfaction with auxiliary constraints in the outer loop of the RS. 
Moreover, it can be extended to handle more complex, non-linear constraints, further broadening the scope of possible applications.



\newpage
\clearpage
\bibliographystyle{ieeetr}
\bibliography{sample-base}
\balance

\newpage
\clearpage
\appendix
\section{Appendix}
\subsection{Proofs}\label{app:proofs}

\subsubsection{Derivations of Equations}
\paragraph{Derivation of Equation (\ref{eq:outer_problem})} 

Observe, that at any iteration $k$ we want to set $M$ optimisation variables $a_{kj}$, where $j=1,...,M$. 
As $\lambda=\lambda_k$ is fixed, we can  simplify Problem (\ref{eq:dualised}) as follows:
\begin{align}
    \max_{a_{kj}\in\{0,1\}}&\Big[ \sum_{k=1}^K \sum_{j=1}^M p_{\theta}(t_j | x, t_{1:k-1}) a_{kj} + \Big(\lambda_k \sum_{k=1}^K \sum_{j=1}^M b(t_j) a_{kj} - \lambda_k R   \Big)\Big] =\nonumber \\
    \max_{a_{kj}\in\{0,1\}}& \sum_{k=1}^K \sum_{j=1}^M \big(p_{\theta}(t_j | x, t_{1:k-1}) + \lambda_k b(t_j)\big) a_{kj},\label{eq:outer_simplified}
\end{align}
where all coefficients $p_{\theta}(t_j | x, t_{1:k-1}) + \lambda_k b(t_j) \geq 0$ and Constraint (\ref{eq:feasibility_constr_2}) holds.
Consequently, setting $a_{kj^*=1}$ for item $t_{j^*}$ with the largest coefficient $p_{\theta}(t_j | x, t_{1:k-1}) + \lambda_k b(t_j)$ maximises Equation (\ref{eq:outer_simplified}).\qed

\vspace{15pt}
\paragraph{Derivation of Equation (\ref{eq: multiplier_update})}

Remember, that the Bregman divergence $D_h(\lambda || \lambda_k)$ is defined as follows:
\begin{align}\label{eq:bregman_divergence}
    D_h(\lambda || \lambda_k) = h(\lambda) - h(\lambda_k) - \frac{\partial h(\lambda_k)}{\partial \lambda} (\lambda - \lambda_k).
\end{align}
From the first order conditions of Problem (\ref{eq:inner_problem_1}) we obtain:
\begin{align}\label{eq:FOC}
    \eta\sum_{j=1}^M \Big( b(t_j) - \frac{R}{K} \Big) a_{kj} + \frac{\partial}{\partial \lambda} D_h(\lambda || \lambda_k) = 0.
\end{align}
Substituting Equation (\ref{eq:bregman_divergence}) into Equation (\ref{eq:FOC}):
\begin{align}\label{eq:eq_17}
    \eta\sum_{j=1}^M \Big( b(t_j) - \frac{R}{K} \Big) a_{kj} + \frac{\partial h(\lambda_{k+1})}{\partial \lambda} - \frac{\partial h(\lambda_k)}{\partial \lambda}=0.
\end{align}
Remember now that we define the Bregman divergence with the specific negative entropy function $h(\lambda) = \lambda \ln \lambda - \lambda$.
This allows us to further simplify Equation (\ref{eq:eq_17}) as follows:
\begin{align}
    \eta\sum_{j=1}^M \Big( b(t_j) - \frac{R}{K} \Big) a_{kj} + \ln \lambda_{k+1} - \ln \lambda_k=0.
\end{align}
Consequently,
\begin{align}
    \ln \lambda_{k+1} = \ln \lambda_k + \ln \exp \Big\{ -\eta\sum_{j=1}^M \Big( b(t_j) - \frac{R}{K} \Big) a_{kj} \Big\},
\end{align}
and, therefore,
\begin{align}\label{eq:eq_one_iter_update}
    \lambda_{k+1} = \lambda_k \exp \Big\{ -\eta\sum_{j=1}^M \Big( b(t_j) - \frac{R}{K} \Big) a_{kj} \Big\}.
\end{align}
Expanding the recursion for $\lambda_k$, we obtain Equation (\ref{eq: multiplier_update}).\qed 

\vspace{15pt}
\subsubsection{Bound on Constraint Violation}\label{app:cv}
We now derive the upper bound on constraint violations reached by Algorithm \ref{alg:online_opt}.
The intuition behind the proof is as follows:
First, observe that we can always choose $\eta$ large enough, so that the relevance term $p_{\theta}(t_j | x, t_{1:k-1})$ in line 3 of Algorithm \ref{alg:online_opt} becomes negligibly small compared to the bid term $\lambda_{k-1} b(t_j)$.
By doing so we can tighten the constraint by uplifting items with the largest bids $b_{max}$.
Consequently, if we compute the minimal number of iterations needed to reach such a large $\eta$, we can lower bound the left hand side of Constraint (\ref{eq:budget_constr}).

\OnlineAlg*
\proof{
First, observe that if the problem is feasible, then it must be that $b_{max}\geq\frac{R}{K}$ as otherwise, target $R$ would have been unachievable.
Now, to prove the current statement, we will derive a lower bound on the sum $\sum_{k=1}^K \sum_{j=1}^M \Big( b(t_j) - \frac{R}{K} \Big) a_{kj}$, see Constraint (\ref{eq:budget_constr}).

We let 
\begin{align}\label{eq:marginal_reward}
    r_k = b(t_{j^*}) - \frac{R}{K}
\end{align} 
be the \textit{marginal reward} contributing to Constraint (\ref{eq:budget_constr}) at iteration $k$ of Algorithm \ref{alg:online_opt} (here, $j^*$ is the index of the allocated item, i.e., $a_{kj^*}=1$).
Clearly, $r_k\geq -\frac{R}{K}$ for all $k=1,...,K$ gives us a trivial lower bound on the per-iteration reward.

We also know that $r_k$ is related to the dual variable $\lambda$ via Equation (\ref{eq: multiplier_update}), i.e., the smaller the dual $\lambda_{k+1}$ the larger was the marginal reward $r_k$ allocated on iteration $k$, see Algorithm \ref{alg:online_opt}.
This allows us to tighten the lower bound on $r_k$ by analysing the change $\lambda_{k+1} - \lambda_k$ at each iteration $k$.

We let $\tilde{k}$ be the last iteration of Algorithm \ref{alg:online_opt} when the constraint violation is smaller or equal to some $\alpha>0$, i.e.,
\begin{align}\label{eq:alpha}
    -\sum_{k=1}^{\Tilde{k}} r_{k} \leq \alpha\; \text{and }\; -\sum_{k=1}^t r_k > \alpha, \;\text{for all } t=\Tilde{k}+1, ..., K.
\end{align}
From Equation (\ref{eq: multiplier_update}) it follows that 
\begin{align}\label{eq:lagrangian_boost}
    \lambda_k \geq \exp\{\eta \alpha\},\;\;\text{for all } k=\tilde{k}+1,...,K.
\end{align}
Now observe, that for any $\beta \in (0,b_{max})$ and for any 
\begin{align}\label{eq:lmbd_k}
    \lambda_k \geq \frac{1}{b_{max} -\beta}
\end{align} 
Algorithm \ref{alg:online_opt} always allocates item $t_j$ with bid $b(t_j)$ at least $\beta$, i.e., $b(t_j)\geq \beta$. 
Indeed, let $\tilde{t}_j$ be an item with the maximal bid  $b(\tilde{t}_j)=b_{max}$. 
Then, from Inequality (\ref{eq:lmbd_k}) for any item $t_j$ we have:
\begin{align}
    \lambda_k (b_{max}-\beta) \geq 1 \geq p_{\theta}(t_j | x, t_{1:k-1}) - p_{\theta}(\tilde{t}_j | x, t_{1:k-1}).
\end{align}
Rearranging, we obtain: 
\begin{align}
    \lambda_k b_{max} + p_{\theta}(\tilde{t}_j | x, t_{1:k-1})  \geq &p_{\theta}(t_j | x, t_{1:k-1}) + \lambda_k \beta \geq \nonumber\\
    &p_{\theta}(t_j | x, t_{1:k-1}) + \lambda_k b(t_j)\\
    &\text{for all}\; t_j: b(t_j) \leq \beta.\nonumber
\end{align} 
Consequently, from line 4 of Algorithm \ref{alg:online_opt} it follows that as long as Inequality (\ref{eq:lmbd_k}) holds, any item $t_j$ with bid $b(t_j)< \beta$ cannot be allocated.
Therefore, if we let $\lambda_k$ on iteration $\tilde{k}$ be at least $1/(b_{max}-\beta)$, then for all consequent iterations we can collect the reward of at least $(K-\tilde{k})(\beta-R/K)$. 
In this case, the total constraint violation must be:
\begin{align}\label{eq:bound} 
    \underbrace{-\sum_{k=1}^{\Tilde{k}} r_{k}}_{\text{$\leq \alpha$, Equation (\ref{eq:alpha})}}  -\sum_{k=\Tilde{k}+1}^{K} r_{k} \leq \alpha - (K-\tilde{k}) \Big(\beta - \frac{R}{K}\Big)\leq
    \beta\tilde{k} - K\Big(\beta - \frac{R}{K}\Big). 
\end{align}
Here, we used the fact that at step $\tilde{k}$ the ongoing constraint violation $\alpha$ cannot exceed $\tilde{k}\cdot\frac{R}{K}$, i.e., $\alpha\leq  \tilde{k}\cdot\frac{R}{K}$.
Now, observe that in order to reach $\lambda_k \geq \frac{1}{b_{max}-\beta}$ (see Inequality (\ref{eq:lmbd_k})) Algorithm \ref{alg:online_opt} needs to make at least $\tilde{k}$ steps, where
\begin{align}
    \tilde{k} = \frac{K \log \lambda_k}{\eta R}=-\frac{K \log(b_{max}-\beta)}{\eta R}.
\end{align}
Consequently, we can bound Inequality (\ref{eq:bound}) further:
\begin{align}
    -\sum_{k=1}^K r_k \leq -\beta\tilde{k} - K\Big(\beta - \frac{R}{K}\Big) \leq &-\beta \frac{K  \log(b_{max}-\beta)}{\eta R} - K\Big(\beta - \frac{R}{K}\Big)=\nonumber\\
    & R - K\beta - \frac{K\beta}{R} \frac{\log(b_{max}-\frac{R}{K})}{\eta}.
\end{align}
Let us now choose $\beta$ so that $K \beta = R$.
We then obtain:
\begin{align}
    -\sum_{k=1}^K r_k \leq -\frac{\log(b_{max}-\frac{R}{K})}{\eta}.
\end{align}
Q.E.D.\qed
}

\vspace{15pt}
\subsubsection{Regret Bound}\label{app:regret}
To derive the bound on the regret reached by Algorithm \ref{alg:online_opt} we follow the general framework that was first outlined in \cite{Balseiro_2023} and then used in \cite{Feng_pacing}, which we generalise for our stochastic temporally consistent setting (Assumption \ref{as:temp_cons}).
To this end, we first compute an upper bound on the expected reward reached by optimal allocation (Statement \ref{th:OPT_bound}). 
Then, in Statement \ref{th:alg_bound} we provide a lower bound on the reward achieved by Algorithm \ref{alg:online_opt}.
Combining these two statements we obtain the final regret bound.

The following upper bound on the optimal reward follows directly from the dualised Problem~(\ref{eq:dualised}), where we use Assumption \ref{as:temp_cons} to substitute the temporally dependent coefficients $p_{\theta}(t_j| x, t_{1:k-1})$ of the constraint optimisation Problem~(\ref{eq:max_usat_online_1}) with the temporally independent $p_{\theta}(t_j | x)$.

\begin{restatable}{statement}{OPTbound}\label{th:OPT_bound}
    Expected reward $\mathbf{E}_{p_{\theta},b}[Reward(OPT)]$ of the optimal allocation is upper bounded by 
    \begin{align}
        \mathbf{E}_{p_{\theta}, b}[Reward(OPT)] \leq \inf_{\lambda\geq 0} K \cdot \mathbf{E}_{p_{\theta}, b} \Big[s_{k}^*(\lambda)\Big] + \frac{2}{3}\sigma_xK^{3/2} \sqrt{2\log\frac{2MK}{\delta}}
    \end{align}
    with probability at least $1-\delta$, where $s_k^*(\lambda) = \max_{j}\{ p_{\theta}(t_j|x) + \lambda (b(t_j) - \frac{R}{K}) \}$.
\end{restatable}
\begin{proof}
    First, from Assumption \ref{as:temp_cons} it follows that for any item $t_j$ its relevance score at position $k$ is
    \begin{align}
        p_{\theta}(t_j|x,t_{1:k-1}) = p_{\theta}(t_j | x) + \sum_{\ell=1}^{k-1} \epsilon_{\ell j}(x), \; k=1,...,K.
    \end{align}
    Let $\mathcal{E}_{kj}(x)=\sum_{\ell=1}^{k-1} \epsilon_{\ell j}(x)$, and observe that  $\mathbf{E}_{\epsilon}[\mathcal{E}_{kj}(x)] = 0$.
    Therefore, for any allocation $a_{kj}$, we have
    \begin{align} \label{eq:tot_rel}
        \sum_{k=1}^K \sum_{j=1}^M p_{\theta}(t_j | x, t_{1:k-1}) a_{kj} = \sum_{k=1}^K \sum_{j=1}^M p_{\theta}(t_j | x) a_{kj} + \sum_{k=1}^K \sum_{j=1}^M \mathcal{E}_{kj}(x) a_{kj}.
    \end{align}

    Observe, that $\sum_{k=1}^K\sum_{j=1}^M \mathcal{E}_{kj}(x) a_{kj} \leq \sum_{k=1}^K\max_{j=1,...,M}\mathcal{E}_{kj}(x)$.
    Notice also that $\mathcal{E}_{kj}(x)\sim\mathcal{N}(0,(k-1)\sigma_x^2)$, and, consequently, $\frac{\mathcal{E}_{kj}(x)}{\sigma_x\sqrt{k-1}} \sim\mathcal{N}(0,1)$.
    We now can bound the tail of $\mathcal{E}_{kj}(x)$:
    \begin{align}
        \Pr( |\mathcal{E}_{kj}(x)| \geq \sigma_x\sqrt{k-1}\cdot t ) \leq 2 e^{-t^2/2}.
    \end{align}
    Let $t=\sqrt{2\log\frac{2MK}{\delta}}$. 
    We then obtain 
    \begin{align}
        \Pr\Big( |\mathcal{E}_{kj}(x)| \geq \sigma_x\sqrt{k-1}\cdot\sqrt{2\log\frac{2MK}{\delta}}\Big) \leq \frac{\delta}{MK}.
    \end{align}
    Now, using the union bound, we obtain:
    \begin{align}
        |\mathcal{E}_{kj}(x)| \leq \sigma_x\sqrt{k-1}\sqrt{2\log\frac{2MK}{\delta}}
    \end{align}
    holds for all $k=1,...,K$, $j=1,...,M$ with probability at least $1-\delta$.
    Thus,
    \begin{align}\label{eq:azuma}
        \sum_{k=1}^K\sum_{j=1}^M \mathcal{E}_{kj}(x) a_{kj} \leq \sum_{k=1}^K\max_{j=1,...,M}|\mathcal{E}_{kj}(x)|\leq \\
        \sigma_x\cdot\sqrt{2\log\frac{2MK}{\delta}}\sum_{k=1}^K \sqrt{k-1}\leq \frac{2}{3}\sigma_x K^{3/2} \sqrt{2\log\frac{2MK}{\delta}} \nonumber
    \end{align}
    with probability at least $1-\delta$.    
    
    Substituting Equation (\ref{eq:azuma}) into Equation (\ref{eq:tot_rel}) we get:
    \begin{align}\label{eq:eq_37}
        \sum_{k=1}^K \sum_{j=1}^M p_{\theta}(t_j | x, t_{1:k-1}) a_{kj} \leq \sum_{k=1}^K \sum_{j=1}^M p_{\theta}(t_j | x) a_{kj} + \frac{2}{3}\sigma_xK^{3/2} \sqrt{2\log\frac{2MK}{\delta}}
    \end{align}
    with probability at least $1-\delta$.
    Notice, that the dualised Problem~(\ref{eq:dualised}) bounds the optimal value of the original Problems~(\ref{eq:max_usat_online_1})-(\ref{eq:feasibility_constr_1}) from above.
    Using Inequality (\ref{eq:eq_37}) and Equation (\ref{eq:dualised}) we obtain:
    \begin{align}\label{eq:reward_opt_bound}
        Reward(OPT)\leq\inf_{\lambda\geq 0}\max_{a_{kj}\in\{0,1\}} \Big[ &\sum_{k=1}^K \sum_{j=1}^M \Big(p_{\theta}(t_j | x) + \lambda\Big(b(t_j)-\frac{R}{K}\Big)\Big) a_{kj} + \\
        &\frac{2}{3}\sigma_xK^{3/2} \sqrt{2\log\frac{2MK}{\delta}}\Big]\nonumber\\
     \text{s.t.,}\;\;&\sum_{k=1}^K a_{kj} \leq 1\;\;\forall j=1,...,M,
    \end{align}
    with probability at least $1-\delta$.
    We define $s_k^*(\lambda) = \max_{j}\{ p_{\theta}(t_j|x) + \lambda (b(t_j) - \frac{R}{K}) \}$, and rewrite Inequality (\ref{eq:reward_opt_bound}) as follows:
    \begin{align}\label{eq:reward_opt_bound_instance}
        Reward(OPT)\leq \inf_{\lambda\geq 0} \sum_{k=1}^K s_{k}^*(\lambda) + \frac{2}{3}\sigma_xK^{3/2} \sqrt{2\log\frac{2MK}{\delta}},
    \end{align}
    with probability at least $1-\delta$. 
    Taking the expectation of Equation (\ref{eq:reward_opt_bound_instance}) 
    concludes the proof. Q.E.D. 
\end{proof}

\vspace{15pt}
We now derive the lower bound on the reward of Algorithm \ref{alg:online_opt}.
To this end, we observe that at each iteration Algorithm \ref{alg:online_opt} chooses the item with the largest score $\sigma(t_j)$ (lines 3-4), which (using Azuma's inequality) we bound by $s_k^*(\lambda_k)$.

\begin{restatable}{statement}{AlgBound}\label{th:alg_bound}
    Expected reward $\mathbf{E}_{p_{\theta},b}[Reward(\mathcal{A})]$ of Algorithm \ref{alg:online_opt} is lower bounded by 
    \begin{align}
        \mathbf{E}_{p_{\theta}, b}[Reward(\mathcal{A})] \geq \mathbf{E}_{p_{\theta}, b}[\sum_{k=1}^K s_k^*(\lambda_k)] - \mathbf{E}_{p_{\theta}, b}[\sum_{k=1}^K \lambda_k r_k] -\sqrt{2K\ln \frac{2}{\gamma}}\nonumber
    \end{align}
    with probability at least $\gamma$.
\end{restatable}
\begin{proof}
    Let $\pi_k$ be the value allocated at iteration $k$ of Algorithm \ref{alg:online_opt}, i.e., $\pi_k = p_{\theta}(t_j | x, t_{1:k-1})$ for $j$, s.t. $ a_{kj}=1$.
    From line 4 of Algorithm \ref{alg:online_opt} it follows that at each iteration $k$ of the algorithm $\pi_k$ must be
    \begin{align}
        \pi_k = &\max_j \Big\{ p_{\theta}(t_j|x, t_{1:k-1}) + \lambda_k \Big(b_j - \frac{R}{K}\Big) \Big\} - \lambda_k r_k =\nonumber \\
        &\max_j \Big\{ p_{\theta}(t_j|x) + \mathcal{E}_{kj}(x) + \lambda_k \Big(b_j - \frac{R}{K}\Big) \Big\} - \lambda_k r_k \geq\nonumber \\
        &\max_j \Big\{ p_{\theta}(t_j|x) + \lambda_k \Big(b_j - \frac{R}{K}\Big) \Big\} + \mathcal{E}_{kj^*}(x) - \lambda_k r_k, 
    \end{align}
    where $j^*\in\arg\max\Big\{p_{\theta}(t_j|x) +  \lambda_k \Big(b_j - \frac{R}{K}\Big)\Big\}$ and $r_k$ is defined in Equation (\ref{eq:marginal_reward}).
    We can now rewrite:
    \begin{align}
        \pi_k \geq s_k^*(\lambda_k) - \lambda_k r_k + \mathcal{E}_{kj^*}(x),
    \end{align}
    and, consequently,
    \begin{align}
        \sum_{k=1}^K \pi_k \geq \sum_{k=1}^K \Big( s_{k}^*(\lambda_k) - \lambda_k r_k\Big) + \sum_{k=1}^K \mathcal{E}_{kj^*}(x) \geq\nonumber \\
        \sum_{k=1}^K \Big( s_{k}^*(\lambda_k) - \lambda_k r_k\Big) - \frac{2}{3}\sigma_x K^{3/2} \sqrt{2\log\frac{2MK}{\gamma}},
    \end{align}
    with probability at least $1-\gamma$
    (here, similarly to Statement \ref{th:OPT_bound} we used Equation (\ref{eq:azuma}) and the union bound).
    Now, taking the expectation of both sides, we obtain:
    \begin{align}
        \mathbf{E}_{p_{\theta}, b}[Reward(\mathcal{A})] = &\mathbf{E}_{p_{\theta}, b}[\sum_{k=1}^K \pi_k] \geq \\
        \mathbf{E}_{p_{\theta}, b}[\sum_{k=1}^K s_k^*(\lambda_k)] - &\mathbf{E}_{p_{\theta}, b}[\sum_{k=1}^K \lambda_k r_k] -\frac{2}{3}\sigma_x K^{3/2} \sqrt{2\log\frac{2MK}{\gamma}}\nonumber
    \end{align}
    with probability at least $1-\gamma$. Q.E.D.
\end{proof}

We can now combine Statement \ref{th:OPT_bound} and Statement \ref{th:alg_bound} to derive the following upper bound on the regret of Algorithm \ref{alg:online_opt}:

\begin{restatable}{lemma}{alg_bound}\label{th:lemma}
    The regret of Algorithm \ref{alg:online_opt} is upper bounded by
    \begin{align}\label{eq:lemma}
        Regret(\mathcal{A}) \leq \mathbf{E}_{p_{\theta}, b}[\sum_{k=1}^K \lambda_k r_k] + \frac{4}{3}\sigma_x K^{3/2} \sqrt{2\log\frac{2MK}{\delta}}
    \end{align}
    with probability at least $1-\delta$. 
\end{restatable}
\begin{proof}
    First, from Statement \ref{th:OPT_bound} it follows that for any $\lambda_k$
    \begin{align}
        K \inf_{\lambda\geq 0} s_k^*(\lambda) \leq \sum_{k=1}^{K} s_k^*(\lambda_k),
    \end{align}

    We can, therefore, rewrite the lower bound on the reward of Algorithm \ref{alg:online_opt} from Statement \ref{th:alg_bound} as follows:
    \begin{align}
        \mathbf{E}_{p_{\theta}, b}[Reward(\mathcal{A})] \geq K&\cdot \mathbf{E}_{p_{\theta}, b}\Big[\inf_{\lambda\geq 0}s_k^*(\lambda)\Big] - \mathbf{E}_{p_{\theta}, b}[\sum_{k=1}^K \lambda_k r_k] \\
        -&\frac{2}{3}\sigma_x K^{3/2} \sqrt{2\log\frac{2MK}{\delta}}
    \end{align}
    with probability at least $1-\delta$.
    Comparing this with Statement \ref{th:OPT_bound}, we conclude:
    \begin{align}
        \mathbf{E}_{p_{\theta}, b}[Reward(OPT)]  - &\mathbf{E}_{p_{\theta}, b}[Reward(\mathcal{A})] \geq \\
        &\mathbf{E}_{p_{\theta}, b}[\sum_{k=1}^K \lambda_k r_k] + \frac{4}{3}\sigma_x K^{3/2} \sqrt{2\log\frac{2MK}{\delta}}\nonumber
    \end{align}
    with probability at least $1-\delta$. Q.E.D.
\end{proof}

Lemma \ref{th:lemma} is intuitive: 
Remember that the Lagrangian multiplier $\lambda$ of Constraint (\ref{eq:budget_constr}) can be thought as the ``shadow price" of the unit of the constraint expressed in the units of Objective (\ref{eq:max_usat_online_1}), \cite{MasColell}.
In this case, if we think of $\lambda_k$ as iterative approximations of $\lambda$, then each iteration of Algorithm \ref{alg:online_opt} imposes the ``cost" of approximately $\lambda_k r_k$ on the objective, and consequently, the total expected cost must be approximately $\mathbf{E}_{p_{\theta}, b}[\sum_{k=1}^K \lambda_k r_k]$, see the first term in Inequality (\ref{eq:lemma}).
The second term in Inequality (\ref{eq:lemma}) bounds the potential value loss due to the temporal consistency assumption. 

We now rely on this Lemma to prove our final regret bound:

\algRegret*
\begin{proof}
    We rely on Lemma \ref{th:lemma}, where we further bound $\sum_{k=1}^K \lambda_k r_k$.
    Observe, that both $r_k$ and $\lambda_k$ are always bounded.
    We let $|r_k|\leq \bar{r}$ and $\lambda_k \leq \bar{\lambda}$, and
    we set $\eta$ so that $|\eta r_k|\leq 1$ From Equation (\ref{eq:eq_one_iter_update}) it follows
    \begin{align}
        \lambda_{k+1} \leq \lambda_k (1-\eta r_k +\eta^2 r_k^2).
    \end{align}
    Consequently,
    \begin{align}
        \lambda_k r_k \leq \frac{\lambda_k-\lambda_{k+1}}{\eta} + \eta \lambda_k r_k^2,
    \end{align}
    and
    \begin{align}
        \sum_{k=1}^K\lambda_k r_k \leq \frac{\lambda_0-\lambda_{K}}{\eta} + \eta\sum_{k=1}^K \lambda_k r_k^2\leq \frac{1}{\eta} + \eta K \bar{\lambda}\bar{r}^2
    \end{align}
    Taking the expectation and substituting this into Lemma \ref{th:lemma} concludes the proof. Q.E.D.
\end{proof}

\subsection{Datasets \& Model Training Details}

\subsubsection{Playlist Ranking}\label{app:playlist_dataset}

Below we provide additional details about our playlist ranking dataset. 
Concretely, we constructed a sequential dataset that matches the requirements of the original Sasrec paper \cite{sasrec_paper}, where the \textit{tokens} predicted by the model correspond to the tracks streamed by users.
Using logged history of user/track interactions we constructed chronologically ordered sequences of 100 tokens of the users' past interactions, and used these sequences to optimise SASRec model for the next-item prediction task. 
We implemented SASRec model following the original paper \cite{sasrec_paper}. 
Specifically, we mapped the input tokens corresponding to track IDs to 80-dimensional embeddings. 
We then passed these embeddings to the two blocks of multi-head attention layers with 5 heads, resulting in approximately 250M parameter model.
We trained the model using negative in-batch sampling with 100 negative labels per each positive label on a cluster of eight A100 GPUs.

\subsubsection{Home Shelves Ranking Ranking}\label{app:home_shelves}
We relied on the original TransAct paper \cite{transact} for the Home Shelves ranking task.
Specifically, we relied on input sequences of up to 128 past users' interactions that included the item URIs as well as a few contextual features (hour, weekday, shelf metadata, region, etc), the type of the item (music, video, etc).
We mapped all input tokens to 64-dimensional embeddings and encoded the sequential history using a 4-layer 4-head Transformer with a 3-hour temporal attention window. 
Our model combines the resulting sequence representation with candidate and context features to predict the probability that a shelf is streamed.
The model has been trained with Adam optimiser and a batch size of 2048 for six epochs and a 90/10 train/validation split. \cite{transact}.

\subsubsection{Amazon e-Commerce Data}\label{app:amazon}
We relied on a 5M user-item interactions subset of \textit{Sports and Outdoors} dataset from \cite{amazon_data}, where we filtered out users with less than five interactions. 
We used Sentence-T5 model \cite{Sentence_T5} to compute the semantic embeddings of the different items using their titles provided in the dataset. 
In this case, each semantic embedding corresponds to a 768-dimensional latent vector which we further processed with Locality-Sensitive Hashing algorithm (LSH) \cite{Charikar_2002} to compute  Semantic IDs of the items.
Each Semantic ID in our experiment corresponds to three 8-bits integer semantic tokens.

To define Constraint (\ref{eq:budget_constr}) we selected two sub-categories of items, namely, \textit{`Camping \& Hiking'} and \textit{`Cycling'} the exposure of which we would like to increase in the recommendation slates. 
These two sub-categories constitute approximately 15\% of the unique items in the dataset.
For each item in our dataset we introduced a binary score $b(t_j)=1$ if item $t_j$ belongs to any of these two sub-categories and $b(t_j)=0$ otherwise.

To build the transformer-based RS we relied on a pre-trained Llama-3.2-1B model that we further fine-tuned using the past sequences of the user-item interactions.
Specifically, for each user we tokenized the user and the item specific information (i.e., items Semantic IDs), and fine-tuned the model on the next token prediction task. 

Observe, that as in this e-Commerce domain the model operates with semantic tokens while the attribute scores $b(t_j)=1$ are defined on the item-level. 
To allow Algorithm \ref{alg:online_opt} process semantic tokens, we inferred the token-level attributes scores from the respective item-level scores. 
Specifically, we let $s_j = [s_j^1,...,s_j^L]$ be Semantic ID of item $t_j$ that consists of $L$ semantic tokens $s_j^1,...,s_j^L$ (in our experiments $L=3$).
We define the token-level attribute score $b(s_j^{\ell})$ of token $s_j^{\ell}$ as the expected value of the Semantic ID that contains this token, i.e., 
\begin{align}
    b(s_j^{\ell}) = \sum_{i=1}^M b(t_i) \cdot \Pr(t_i | s_j^1,...,s_j^{\ell}),
\end{align}
where
\begin{align}
    \Pr(t_i | s_j^1,...,s_j^{\ell}) = \frac{1}{\sum_{k=1}^M \mathbf{1}_{\{s_k^1=s_j^1,...,s_k^{\ell}=s_j^{\ell}\}}}.
\end{align}
Intuitively, when we allocate the very first token $s_{j}^1$ of a certain Semantic ID, we look into all items $t_j$ than may be generated from this token, and average their expected attribute score.
Our results in Section \ref{sec:offline_evaluation} illustrate that primal-dual decoding is equally efficient in the semantic ID space as in the regular item ID space.



\end{document}